\documentclass{article}
\usepackage{amsmath}
\usepackage{amsthm}
\newtheorem{theorem}{Theorem}
\usepackage{times}
\usepackage{graphicx} 
\usepackage{subfigure} 

\usepackage{natbib}

\usepackage{hyperref}

\begin{document} 

\title{Learning Local Invariant Mahalanobis Distances}

\date{}
\author{Ethan Fetaya \\
 Weizmann Institute of Science \\
 \and 
 Shimon Ullman \\
 Weizmann Institute of Science \\}


\maketitle
\begin{abstract} 
For many tasks and data types, there are natural
transformations to which the data should be
invariant or insensitive. For instance, in visual
recognition, natural images should be insensitive
to rotation and translation. This requirement and
its implications have been important in many machine
learning applications, and tolerance for image
transformations was primarily achieved by
using robust feature vectors. In this paper we propose a novel and computationally
efficient way to learn a local Mahalanobis
metric per datum, and show how we can
learn a local invariant metric to any transformation
in order to improve performance.
\end{abstract} 

Metric learning is a machine learning task which learns
a distance metric $d(x,y)$ between data points, based on data instances. As distances play an important role in many machine learning algorithms, e.g. k-Nearest Neighbor and k-Means clustering, finding an appropriate metric for the task can improve performance considerably. This approach has been applied successfully to  many problems such as face identification \cite{faceId}, image retrieval \cite{imRet,deepRank}, ranking \cite{ranking} and clustering \cite{clustering} to name just a few.\\

A standard approach to metric learning is to learn a global Mahalanobis metric 
\begin{equation}
d(x,y)^2_M=(x-y)^TM(x-y)
\end{equation}
Where $M$ is a positive semi-definite matrix (PSD). The PSD constraint only assures this is a pseudometric , but for simplicity we will not make this distinction. Various algorithms \cite{lmnn,eqConst,infoMet} differ by the objective through which they learn the matrix $M$ from the data. As $M$ is a PSD matrix, it can  be written as $M=L^TL$ and therefore
\begin{align*}
&d(x,y)^2_M=(x-y)^TM(x-y)=||\tilde{x}-\tilde{y}||^2_2\\
&\tilde{x}=Lx,\,\,\tilde{y}=Ly.
\end{align*}
This means that finding an optimal Mahalanobis distance is equivalent to finding the optimal linear transformation on the data, and then using $L_2$ distance on the transformed data. This approach has two limitations, first it is limited to linear transformation. Second, it requires a large amount of labeled data.\\

One approach that can be used to overcome the first limitation is to use local distances \cite{local} where we learn a unique distance function per training datum.  Local approaches do not produce, in general, a global metric (as they are usually not symmetrical) but are commonly considered metric learning nonetheless. These methods, in general, need similar and dissimilar training data for each local metric.  \\

In our current work we use a local approach inspired by the work on exemplar-SVM \cite{exemplar}, that showed that using only negative examples can suffice for good performance. The intuition behind this is that objects of the same class do not necessarily have to be similar, but objects from different classes must be dissimilar. We will show how to learn a local Mahalanobis distance that for each datum tries to keep the non-class as far away as possible. This approach can use a large amount of weakly supervised data, as in many cases negative examples are easier then positive examples to acquire. For example, if we are interested in face identification, we can learn a local metric around a query face image given a bank of train face images, which we only assume do not belong to the queried person. Unlike other metric learning methods, we will not need any labels on which image belongs to which person in the negative set.\\

The intuition why Mahalanobis distances are the natural model for local metrics  is simple. Assume we have some metric $d(x,y)$ on the dataset and assume that it is smooth (at least continuously twice differentiable). From the metric properties we know that if we fix $x$ and look at $f(y)=d(y,x)$ then $f$ has a global minimum at $y=x$. Applying second order Tylor approximation to $f$ around $x$ we get
\begin{equation}
\begin{split}
&d(y,x)=f(y)\approx f(x)+(y-x)^T\nabla f(x)+\\&(y-x)^T\nabla^2f(x)(y-x)=(y-x)^T\nabla^2f(x)(y-x)
\end{split}
\end{equation}
The equality holds since $x$ is the global minimum with value $f(x)=d(x,x)=0$, and this also implies that $\nabla^2f(x)$ is positive semidefinite. While the Taylor approximation only holds for values of $y$ close to $x$, as metric methods such as k-NN focus on similar objects the approximation should be good at the points of interest. This observation leads us to look for local matrices that are of the form of a Mahalanobis distance. \\

We will first define our local Mahalanobis distance learning method as a semidefinite programming problem. We will then show how this problem can be solved efficiently without any costly matrix decompositions. This allows us to solve high dimensional problems that regular semidefinite solvers cannot handle.\\

The second major contribution of this paper will be to show how invariant local matrices can be learned. In many cases we know there are simple transformations that our metric should not be sensitive to. For example, small translation and rotation on natural images. We know a priori that if $x'=T(x)$, where $T$ is the said transformation, then $d(x,x')\approx 0$. We will  show how this prior knowledge about our data can by incorporated by learning a local invariant metric. This also can be done in an efficient manner, and we will show that this improves performance in our experiments.

\section{Related work}
Metric learning is an active research field with many algorithms, generally divided into linear \cite{lmnn} which learn a Mahalanobis distance, non-linear \cite{nonLinear} that learn a nonlinear transformation and use $L_2$  distance on the transformed space, and local which learn a metric per datum. 
The LMNN and MLMM \cite{lmnn} algorithm  are considered the leading  metric learning method. For a recent comprehensive survey that covers linear, non-linear and local methods see \cite{survey}. \\

The exemplar-SVM algorithm \cite{exemplar} can be seen as a local similarity measure. This is obtained by maximizing margins, with a linear model, and is weakly supervised as our work. Unlike exemplar-SVM, we learn a Mahalanobis matrix and can learn an invariant metric.
Another related work is PMLM \cite{PMLM}, which also finds a local Mahalanobis metric for each data point. However, this method uses global constraints, and therefore cannot work with weakly supervised data, i.e. a single positive example. All the techniques above do not learn local invariant metrics.\\

The most common way to achieve invariance, or at least insensitivity, to a transformation  in computer vision applications is  by using hand-crafted descriptors such as SIFT \cite{sift} or HOG \cite{hog}. Another way, used in convolutional networks \cite{cnn}, is by adding pooling  and subsampling forcing the net to be insensitive to small translations. It is important to note that transformations such as rotations have a global behaviour, i.e. there is a global consistency between the pixel movement. This global consistency is not totally captured by the pooling and subsampling. As we will see in our experiments, using an invariant metric can be useful even when working with robust features such as HOG.

\section{Local Mahalanobis}
In this section we will show how a local Mahalanobis distance with maximal margin can be learned in a fast and simple way.\\

We will assume that we are given a single query image $x_0$ that belong to some class, e.g. a face of a person. We will also be given a set of negative data $x_1,...,x_N$ that do not belong to that class, e.g. a set of face images of various other people. We will learn a local Mahalanobis metric for $x_0$, $M(x_0)\succeq 0$, where $M\succeq 0$ means $M$ is positive semi-definite. For matrices  $M,N$, we will denote by $||M||$ the Frobinous norm $||M||^2=\sum_{ij}M_{ij}^2$ and by $\left<M,N\right>$ the standard inner product $\left<M,N\right>=\sum_{ij}M_{ij}N_{ij}$.\\

We wish to find a Mahalanobis matrix $M$ given the positive datum $x_0$ and the negative data $x_1,...,n_n$. Large margin methods have been very successful in metric learning \cite{lmnn}, and more generally in machine learning, therefore, our algorithm will look for the PSD matrix $M$ that maximizes the distance to the closest negative example
\begin{equation}
M=\arg\max_{M\succeq 0}(\min_{1\leq i\leq n}(x_i-x_0)^TM(x_i-x_0))
\end{equation}

The optimization cannot be solved as it is not bounded, since multiplying $M$ by a scalar multiplies the minimum distance by the same scalar. This can be solved by normalizing $M$ to have $||M||=1$. As normally done with margin methods, we can minimize the norm under fixed margin constrained instead of maximizing the margin under fixed norm constraint. The  resulting objective is  

\begin{equation}\label{noRelax}
\begin{split}
M(x_0) &=\arg\min_M \frac{1}{2}||M||^2 \\
subject\,\, to :\,\,& (x_i - x_0)^TM(x_i - x_0) \geq 2\,\,\,\,\forall i \in \{1,..., n\} \\
&M \succeq 0
\end{split}
\end{equation}
Where the constant $2$ is arbitrary and will be convenient later on. While this is a convex semidefinite programming task, it is very slow for reasonable dimensional data (in the thousands)  even for state of the art solvers. This is because PSD solvers apply a projection to the semidefinite cone, performing an expensive singular value or eigen decomposition at each iteration.\\

To solve this optimization in a fast manner we will first relax the PSD constraint and look at the following objective 
\begin{equation}\label{relax}
\begin{split}
M(x_0) &=\arg\min_M \frac{1}{2}||M||^2 \\
subject\,\, to :\,\,& (x_i - x_0)^TM(x_i - x_0) \geq 2\,\,\,\,\forall i \in \{1,..., n\}
\end{split}
\end{equation}
We will then see how this is equivalent to a kernel SVM problem with a quadratic kernel, and therefore can be solved easily with off-the-shelf SVM solvers such as LIBSVM \cite{libsvm}. Finally we will show how the solution of objective \ref{relax} is in fact the solution of objective \ref{noRelax} resulting in a fast solution to objective \ref{noRelax} without any matrix decomposition.

\begin{theorem}\label{trmSVM}
The solution of objective \ref{relax} is given by running kernel SVM with kernel $k(x,y)=\left<x,y\right>^2$ on inputs $\tilde{x}_0,\tilde{x}_1,...,\tilde{x}_n$ where $\tilde{x}_i= x_i-x_0$ 
\end{theorem} 
\begin{proof}
Define $\varphi(x)=x\cdot x^T$, a function that maps a column vector to a matrix. This function has the following simple properties:
\begin{itemize}
\item $k(x,y)=\left<x,y\right>^2=\left<\varphi(x),\varphi(y)\right>$,  i.e.  the function $\varphi$ is the mapping associated with the quadratic kernel.
\item For any matrix $W$, we have $\left<W,\varphi(x)\right>=x^TWx$.
\end{itemize}  
which can be easily verified using $\varphi(x)_{ij}=x_ix_j$. We can define auxiliary labelling  $y_0 = -1$ and $y_i=1$ for $1\leq i \leq n$. Combining everything objective \ref{relax} can be rewritten as 
\begin{equation}\label{svm}
\begin{split}
M(x_0) &=\arg\min_M \frac{1}{2}||M||^2 \\
subject\,\, to :\,\,& y_i\cdot\left(\left<M,\varphi(\tilde{x}_i)\right> -1\right)\geq 1\,\,\,\,\forall i \in \{0,1,..., n\}
\end{split}
\end{equation}
where we can include $i=0$ as $\left<M,\varphi(\tilde{x}_0)\right>=0$ for any matrix. \\

Objective \ref{svm} is exactly an SVM problem with quadratic kernel, with bias fixed to one,  given inputs $\tilde{x}_0,...,\tilde{x_n}$, proving the theorem. Notice also that for the identity matrix $M=I$ we have $\left<M,\varphi(\tilde{x_i})\right>>0$ for $i\geq 1$ and $\left<M,\varphi(\tilde{x_i})\right>=0$ for $i=0$, therefore the data is separable by $M=I$ and the optimization is feasible.
\end{proof}

Now that we have shown how objective \ref{relax} can be converted into a standard SVM form, for which efficient solvers exists, we will show how it is the solution to objective \ref{noRelax}.

\begin{theorem}\label{trmRLX}
The solution to objective \ref{relax} is the solution to objective \ref{noRelax}.
\end{theorem}
\begin{proof}
To prove the theorem it suffices to show that the solution is indeed positive semidefinite. A well known observation arrising from the dual formulation of the SVM objective \cite{svmTut} is that the optimal solution $M$ has the form \begin{equation}
M=\sum_{i=0}^n\alpha_iy_i\varphi(\tilde{x}_i),\quad \alpha_i\geq 0.
\end{equation} Since $\varphi(x)\succeq 0$ for any $x$, as its only nonzero eigenvalue is $||x||$,  $\varphi(\tilde{x}_0)=\varphi(0)=0$, and $y_i=1$ for $i\geq1$  we get
\begin{equation}\label{9}
M=\sum_{i=0}^n\alpha_iy_i\varphi(\tilde{x}_i)=\sum_{i=1}^n\alpha_i\varphi(\tilde{x}_i)\succeq 0,
\end{equation}
where the positive semidefiniteness in eq. \ref{9} is assured  due to the set of PSD matrices being a convex cone.
\end{proof}

Combining theorem \ref{trmSVM} with theorem \ref{trmRLX} we get that in order to solve objective \ref{noRelax} it is enough to run an SVM solver with a quadratic  kernel function, thus avoiding any matrix decomposition. \\

Looking at this as a SVM problem has further benefits. The SVM solvers do not compute $M$ directly, but return the set of support vectors $\tilde{x}_{i_1},...,\tilde{x}_{i_k}$ and coefficients $\alpha_{i_1},...,\alpha_{i_k}$ such that $M=\sum_k\alpha_{i_k}\varphi(\tilde{x}_{i_k})$. This allows us to work in high dimension $d$, where the $\mathcal{O}(d^2)$ memory needed to store the matrix can be a problem, and can slow computations further. As the rank of the matrix is bounded by the number of support vectors, one can see that in many applications we get a relatively low rank matrix. This bound on the rank can be improved by using sparse-SVM algorithms \cite{ssvm}. In practice we got low rank matrices without resorting to sparse $SVM$ solvers.

\section{Local Invariant Mahalanobis}
For some applications, we know a priori that certain transformations should have a small effect on the metric. We will show how to include this knowledge into the local metric we learn, learning locally invariant metrices. In section \ref{experiments} we will see this has a major effect on performance. \\

Assume we know a set of functions $T_1,...,T_k$ that the desired metric should be insensitive to, i.e. $d(x,T_i(x))$ should be small for all $x$ and $i$. A canonical example is small rotations and translations on natural images. One of the major issues in computer vision arises from the instability of the pixel representation to these transformations. Various descriptors such as SIFT \cite{sift} and HOG \cite{hog} offer a more robust representation, and have been highly successful in many computer vision applications. We will show in section \ref{experiments} that even when using a relatively robust representation such as HOG, learning an invariant metric has a significant impact.\\

A natural way to mathematically formulate the idea of being insensitive to a transformation, is to require the leading term of the approximation to vanish in that direction. In our case this means
\begin{equation}
(T(x_0)-x_0)^T\nabla^2_yd(x_0,y)(T(x_0)-x_0)=0.
\end{equation}

If we return to our basic intuition of the local Mahalanobis matrix as the Hessian matrix $\nabla^2_yd(x_0,y)$, we can now state the new local invariant Mahalanobis objective
\begin{equation}\label{invariant}
\begin{split}
M(x_0&) =\arg\min_M \frac{1}{2}||M||^2 \\
sub&ject\,\, to :\\ &(x_i - x_0)^TM(x_i - x_0) \geq 2\,\,\,\,\forall i \in \{1,..., n\} \\
&(T_j(x_0) - x_0)^TM(T_j(x_0) - x_0)=0\,\,\,\,\forall j \in \{1,..., k\} \\
&M \succeq 0
\end{split}
\end{equation}

We will show how by applying a small transformation to the data, we can reduce this to objective \ref{noRelax} which we can solved easily.

\begin{theorem}
Define $V=span\{T_1(x_0)-x_0,...,T_k(x_0)-x_0\}$, then the minimizer of objective \ref{noRelax} with $x_i-x_0$ replaced by $z_i$, its projection to $V^\perp$ is the minimizer of objective \ref{invariant}.
\end{theorem}
\begin{proof}
For PSD matrices, the constraint that $(T_j(x_0) - x_0)^TM(T_j(x_0) - x_0)=0$ is equivalent to $M(T_j(x_0) - x_0)=0$. This can be seen if we write the vector in the basis of $M$ eigenvectors, and notice that components with positive eigenvalues have a positive contribution to the quadratic form. This means that $Mv=0$ for all $v\in V$. Each vector $x_i-x_0$ can be split into two orthogonal elements, $x_i-x_0=z_i+v_i$ where $v_i$ is its projection onto $V$ and $z_i$ is its projection onto $V^\perp$. Our equality constraints $(T_j(x_0) - x_0)^TM(T_j(x_0) - x_0)=0$ now imply
\begin{equation}
(x_i-x_0)^TM(x_i-x_0)=(z_i+v_i)^TM(z_i+v_i)=z_i^TMz_i
\end{equation}
since all the other terms vanish. We can now rewrite objective \ref{invariant} as

\begin{equation}
\begin{split}
M(x_0) &=\arg\min_M \frac{1}{2}||M||^2 \\
subject\,\, to :\,\,& z_i^TMz_i \geq 2\,\,\,\,\forall i \in \{1,..., n\} \\
& Mv=0\,\,\,\,\forall v \in V\\
&M \succeq 0
\end{split}
\end{equation}
If we forget the equality constrains we get objective \ref{noRelax} with $x_i-x_0$ replaced by $z_i$. To finish the proof we need to show that the solution to the optimization without the equality constraints, does indeed satisfy them. \\

As we have already seen in the proof of theorem \ref{trmRLX} the optimal solution is of the form $M=\sum\alpha_i\varphi(z_i)=\sum\alpha_iz_i\cdot z_i^T$. The vector $z_i$ is a member of $V^T$ so  for $v\in V$
\begin{equation}
Mv=\left(\sum\alpha_iz_i\cdot z_i^T\right)v=\sum\alpha_iz_i\cdot (z_i^Tv)=0
\end{equation}
Proving that the solution satisfies the equality constraints.
\end{proof}

A few comments are worth noting about this formulation. First the problem may not be linearly separable, although in our experiments with real data we did not encounter any unseperable case. This can be easily solved, if needed, by the standard method of adding slack variables. Second, the algorithm just adds a simple preprocessing step to the previous algorithm and runs in approximately the same time. 
\section{Experiments}\label{experiments}
\subsection{Running time}
We compared running the optimization with an SVM solver \cite{libsvm}, to solving it as a semidefinite problem and as a quadratic problem (relaxing the semidefinite constraint). The main limitation when running off-the-shelf solvers is memory. Quadratic or semidefinite solvers need a constraint matrix, which in our case is a full matrix of size $n\times d^2$ where $n$ is the number of samples and $d$ is the data dimension. We tested all three approaches on the MNIST dataset of dimension 784 using only 5000 negative examples, as this already resulted in a matrix of size 24.6Gb. \\

Currently first order methods, such as ADMM  \cite{ADMM}, are the leading approaches to solving problems such as quadratic and semidefinite programming for large matrices. We used YALMIP for modeling and solved using SCS \cite{scs}. The time to run this as an semidefinite program was $1152\pm417sec$. The time it took to run this as a quadratic program was $545\pm74sec$. In comparison, when we run this as an SVM problem it took at most  $0.36sec$. We excluded the time needed to build the $n\times d^2$ constraint matrix for the quadratic and semidefinite solvers. \\

This order of magnitude improvement should not be a surprise. It is a well known that while SVM can be solved as a quadratic program, generic quadratic solvers perform much slower then solver designed specifically for SVM. \\
\subsection{MNIST}
\begin{table}
\caption{Classification error for MNIST dataset.}
\label{mnistErr}
\vskip 0.15in
\begin{center}
\begin{small}
\begin{sc}
\begin{tabular}{lcccr}
\hline
Method & Error\\
\hline
eSVM    & 1.75\%  \\
eSVM+shifts & 1.59\% \\
Local Mahal    & 1.69\%  \\
quadSVM+shifts    & 1.50\%     \\
inv-Mahal (our method) & 1.26\%   \\
LMNN     & 1.69\% \\
MLMNN      & 1.18\%  \\
   \\
\hline
\end{tabular}
\end{sc}
\end{small}
\end{center}
\vskip -0.1in
\end{table}
The MNIST dataset is a well known digit recognition dataset, comprising of  $28\times 28$ grayscale images on which we perform deskewing preprocessing. For each of the $60,000$ training images we computed a local Mahalanobis distance and local invariant Mahalanobis(using only negative examples). On test time we performed $knn$ classification with $k=3$ using the local metrics. We show some examples of nearest neighbours in Figure \ref{mnistFig}. We compared this with exemplar-SVM, as the leading technique most similar to ours. We also compared our scores to exemplar-SVM where we add the tansformed images as positive training data. To show the importance of the invariance objective, we compare also to SVM with quadratic kernel to which we add the transformed data as positive training data (unlike the way we use the shifted data). Finally, we compared our results to the state-of-the-art metric learning LMNN method (linear metric), and to MLMNN,  a local version of LMNN, which learns multiple metrics (but not one per datum). \\

As can be seen in table \ref{mnistErr}, we perform much better then exemplar SVM and are comparable with MLMNN. It is important to note that unlike MLMNN, we compare each datum only to negatives, so our methods is applicable in scenarios where MLMNN is not. \\

\begin{figure}[h!]\label{mnistFig}
  \centering
    \includegraphics[width=0.3\textwidth]{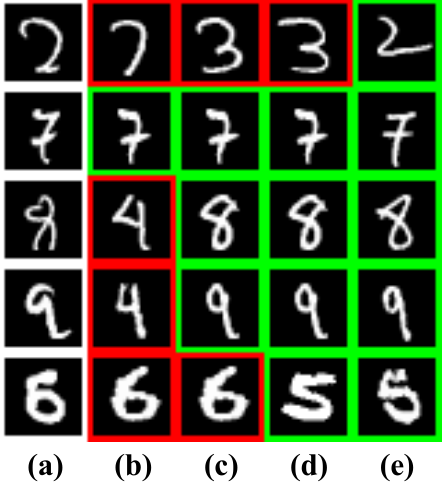}
      \caption{Nearest neighbour for various matrices. (a) original image (b) $L_2$ distance (c) exemplar-SVM (d) local-Mahalanobis (e) local invariant Mahalanobis}
\end{figure}

Another key observation is the difference between the invariant-Mahalanobis and the quadratic-SVM with shifts. While very similar functionally, we see that looking at the problem as a local Mahalanobis matrix gives important intuition, i.e. the way to use the shifted images, that leads to better performance.

\subsection{Labeling faces in the wild (LFW)}
LFW is a challenging dataset containing 13,233 face images of 5749 different individuals with a high level of variability. The LFW dataset is divided into 10 subsets, when the task is to classify 600 pairs of images from one subset to same/not-same using the other 9 subsets as training data. We perform the unsupervised LFW task, where we do not use any labelling inside the training images we get, besides the fact that they are different than both test images. \\

We used the aligned images \cite{deepFunnel} and represented using HOG features \cite{hog}. We compared our results to a cosine similarity baseline, to exemplar-SVM and exemplar-SVM with shifts. We note that we cannot use LMNN or MLMNN on this data, as we only have negative images with a single positive image.

\begin{table}[h]
\caption{Classification error for LFW dataset.}
\label{LFWerr}
\vskip 0.15in
\begin{center}
\begin{small}
\begin{sc}
\begin{tabular}{lcccr}
\hline
Method & Error\\
\hline
cosine similarity & 30.57$\pm$ 1.4\% \\
eSVM    & 26.90$\pm$2.2\%  \\
eSVM+shifts & 27.12$\pm$2.3\% \\
Local Mahal    & 19.85$\pm$1.3\%  \\
inv-Mahal (our method) & 19.48$\pm$1.5\%   \\
   \\
\hline
\end{tabular}
\end{sc}
\end{small}
\end{center}
\vskip -0.1in
\end{table}

As we can see from table \ref{LFWerr}, the local Mahalanobis greatly out-performs the exemplar-SVM. We also see that even when using robust features such as HOG, learning an invariant metric improves performance, albeit to a lesser degree.
\section{Summary}
We showed an efficient way to learn a local Mahalanobis metric given a query datum and a set of negative data points. We have also shown how to incorporate prior knowledge about our data, in particular the transformations to which it should be robust, and use it to learn locally invariant metrics. We have shown that our methods are competitive with leading methods while being applicable to other scenarios where methods such as LMNN and MLMNN cannot be used. 
\bibliography{local_inv_arxiv}
\bibliographystyle{alpha}

\end{document}